\documentclass{article}

\usepackage{microtype}
\usepackage{listings}
\usepackage{color}
\pdfoutput=1
\usepackage[usenames,dvipsnames,svgnames,table]{xcolor}
\usepackage{graphicx}
\usepackage{subfigure}
\usepackage{booktabs} 

\usepackage{hyperref}

\usepackage[accepted]{icml2018}
\icmltitlerunning{Evaluating Variational Inference}

\usepackage[toc,page]{appendix}
\graphicspath{{./fig/}}
\pdfoutput=1

\usepackage{relsize}
\usepackage{fancyvrb}
\usepackage{amsmath}
\usepackage{amssymb}
\usepackage{amsthm}
\DeclareMathOperator{\E}{\mathrm{E}}

\DeclareMathOperator{\KL}{\mathrm{KL}}
\DeclareMathOperator{\R}{\mathbb{R}}
\DeclareMathOperator{\N}{\mbox{N}}

\newtheorem{proposition}{Proposition} 
\usepackage{accents}

\newcommand{\doublehat}[1]{\hat #1}

\definecolor{codegreen}{rgb}{0,0.6,0}
\definecolor{codegray}{rgb}{0.5,0.5,0.5}
\definecolor{codepurple}{rgb}{0.58,0,0.82}
\definecolor{backcolour}{rgb}{0.99,0.99,0.97}
\lstdefinestyle{stan}{
	literate={~}{$\sim$}{1},
	backgroundcolor=\color{backcolour},   
	commentstyle=\color{codegreen},
	otherkeywords = {real, vector, matrix, data, model, parameters, transformed},
	keywordstyle=\color{magenta},
	numberstyle=\tiny\color{codegray},
	stringstyle=\color{codepurple},
	emph={%
		normal, cauchy, inv_gamma, bernoulli_logit, gamma%
	},
	emphstyle=\color{codepurple},%
	basicstyle={\footnotesize,\ttfamily},
	breakatwhitespace=false,         
	breaklines=true,                 
	captionpos=b,                    
	keepspaces=true,                 
	numbers=left,                    
	numbersep=5pt,                  
	showspaces=false,                
	showstringspaces=false,
	showtabs=false,                  
	tabsize=2
}
\begin{document}

\twocolumn[
	\icmltitle{Yes, but Did It Work?: Evaluating Variational Inference}

\begin{icmlauthorlist}
	\icmlauthor{Yuling Yao}{Columbia}
	\icmlauthor{Aki Vehtari }{Aalto}
	\icmlauthor{Daniel Simpson}{Toronto}
	\icmlauthor{Andrew Gelman}{Columbia}
\end{icmlauthorlist}

\icmlaffiliation{Columbia}{Department of Statistics, Columbia University, NY, USA}
\icmlaffiliation{Aalto}{Helsinki Institute for Information Technology, Department of Computer Science, Aalto University,  Finland}
\icmlaffiliation{Toronto}{Department of Statistical Sciences, University of Toronto, Canada}
\icmlcorrespondingauthor{Yuling Yao}{yy2618@columbia.edu}

\icmlkeywords{Variational Inference, diagnostics, Pareto smoothed importance sampling,  Bayesian approximation}

\vskip 0.3in
]
\printAffiliationsAndNotice{}
\begin{abstract}
While it's always possible to compute a variational approximation to a posterior distribution, it can be difficult to discover problems with this approximation. We propose two  diagnostic algorithms to alleviate this problem. The  Pareto-smoothed importance sampling (PSIS) diagnostic  gives a goodness of fit measurement for joint distributions, while simultaneously improving the error in the estimate. The variational   simulation-based calibration (VSBC) assesses the average performance  of point estimates.
\end{abstract}
\section {Introduction}

 Variational Inference (VI), including a large family of posterior approximation methods like stochastic VI (\citealt{hoffman2013stochastic}), black-box VI (\citealt{ranganath2014black}), automatic differentiation VI (ADVI, \citealt{kucukelbir2017automatic}), and many other variants, has emerged as a widely-used method for scalable Bayesian inference. These methods come with few theoretical guarantees and it's difficult to assess how well the computed variational posterior approximates the true posterior.
 
 
Instead of computing expectations or sampling draws from the posterior $p(\theta \mid y)$, variational inference fixes a family of  approximate densities $\mathcal{Q}$, and finds the member $q^*$ minimizing the Kullback-Leibler (KL) divergence to the true posterior: $ \KL\left(  q(\theta), p(\theta
\mid y)\right). $
This is equivalent to maximizing the {evidence lower bound} (ELBO): 
\begin{align}
 \quad \mathrm{ELBO}(q)=  \int_\Theta \left(  \log p(\theta,y)-\log q(\theta) \right) q(\theta) d\theta. 
\end{align}\label{eqn:objecitive}
  
There are many situations where the VI approximation is flawed. This can be due to the slow convergence of the optimization problem, the inability  of the approximation family to capture the true posterior, the asymmetry of the true distribution, the fact that the direction of the KL divergence under-penalizes approximation with too-light tails, or all these reasons.   We need a diagnostic algorithm to test whether the VI approximation is useful.

There are two levels of diagnostics for variational inference. First the convergence test should be able to tell if the objective function  has converged to a local optimum. When the optimization problem \eqref{eqn:objecitive} is solved through  stochastic gradient descent (SGD),  the convergence can be assessed  by monitoring the running average of ELBO changes. Researchers have introduced many convergence tests  based on  the asymptotic property of  stochastic approximations \citep[e.g.,][]{sielken1973stopping, stroup1982new, pflug1990non, wada2015stopping, chee2017convergence}.  Alternatively, \citet{blei2017variational} suggest monitoring the expected log predictive density by holding out an independent test dataset. 
After convergence,  the optimum  is still an approximation to the truth.  This  paper is focusing on the second level of VI diagnostics whether the  variational posterior $q^*(\theta)$ is close enough to the true posterior $p(\theta|y)$ to be used in its place.

Purely relying on the objective function or the equivalent ELBO does not solve the problem.  An unknown multiplicative   constant  exists in $ p(\theta, y) \propto  p(\theta \mid y) $ that  changes  with reparametrization,  making it meaningless to compare ELBO across two  approximations.  Moreover, the ELBO is a quantity on an uninterpretable scale, that is it's not clear at what value of the ELBO we can begin to trust the variational posterior. This makes it next to useless as a method to assess how well the variational inference has fit.

In this paper we propose two diagnostic methods that assess, respectively, the quality of the entire variational posterior for a particular data set, and the average bias of a point estimate produced under correct model specification. 

The first method is based on generalized Pareto distribution diagnostics used to assess the quality of a importance sampling proposal distribution in Pareto smoothed importance sampling \citep[PSIS,][]{vehtari2015pareto}.   The benefit of PSIS diagnostics is two-fold.  First, we can tell the  discrepancy between the approximate and the true distribution by the estimated continuous $\hat k$ value. When it is larger than a pre-specified  threshold, users should be alert of the limitation of current variational inference computation and consider further tuning it or turn to exact sampling like Markov chain Monte Carlo (MCMC). Second, in the case  when $\hat k$  is small,  the fast convergence rate of the importance-weighted Monte Carlo integration guarantees a better estimation accuracy.  In such sense, the  PSIS diagnostics could also be viewed as a post-adjustment for VI approximations.  Unlike  the  second-order correction \citet{giordano2017covariances}, which relies  on an un-testable unbiasedness assumption,
 we make diagnostics and adjustment at the same time.

The second diagnostic considers only the quality of the median of the variational posterior as a point estimate (in Gaussian mean-field VI this corresponds to the modal estimate). This diagnostic assesses the  average behavior of the point estimate under data from the model and can indicate when a systemic bias is present. The magnitude of that bias can be monitored while computing the diagnostic.  This diagnostic can also assess the average calibration of univariate functionals of the parameters, revealing if the posterior is under-dispersed, over-dispersed, or biased. This diagnostic could be used as a partial justification for using the second-order correction of  \citet{giordano2017covariances}.

\section{Is the Joint Distribution Good Enough? }
If we can draw a sample $(\theta_1, \dots, \theta_S)$ from $p(\theta | y)$,  the expectation of any integrable function  $\E_{p} [h(\theta)]$ can be estimated by Monte Carlo integration:
$  \sum_{s=1}^{S} h(\theta_s) / S\! \! \!\xrightarrow[]{ \enskip S\to \infty \enskip } \!\!\! \E_p \left[h(\theta)\right]. $
Alternatively, given samples $(\theta_1, \dots, \theta_S)$ from a proposal distribution $q(\theta)$,
the \emph {importance sampling} (IS) estimate   is
$ \left( {\sum_{s=1}^S h(\theta_s)r_s} \right) / {\sum_{s=1}^S   r_s} $,
where the importance ratios $r_s$ are defined as 
\begin{align} \label{ratios}
r_s=\frac{p(\theta_s, y)}{q(\theta_s)}.
\end{align}

In general,  with a sample $(\theta_1, \dots, \theta_S)$ drawn from the variational posterior $q(\theta)$, we consider a family of estimates with the form 
\begin{align}\label{eqn:is}
E_p [h(\theta)]\approx \frac{ \sum_{s=1}^S h(\theta_s) w_s }{  \sum_{s=1}^S w_s},  
\end{align}
which contains two extreme cases:
\begin{enumerate}
	\item When  $w_s  \equiv 1$, estimate \eqref{eqn:is} becomes the plain VI estimate  that is we completely trust the VI approximation. In general, this will be biased to an unknown extent and inconsistent. However, this estimator has small variance.
	\item When $w_s=r_s$, \eqref{eqn:is} becomes \emph{importance sampling}. The  strong law of large numbers ensures it is consistent as $S\to \infty$, and with small $O(1/S)$ bias due to self-normalization. But the IS estimate may have a large or infinite variance.  
\end{enumerate}

There are two questions to be answered. First, can we find a better bias-variance trade-off than both plain VI and IS? 

Second,  VI approximation $q(\theta)$ is not designed for an optimal IS proposal, for it has a lighter tail than $p(\theta|y)$ as a result of  entropy penalization, which lead to a heavy right tail of $r_s$.  A few large-valued $r_s$ dominates the summation, bringing in large uncertainty. But does the finite sample performance of IS or stabilized IS contain the information about the dispensary measure between  $q(\theta)$ and $p(\theta|y)$?

\subsection{Pareto Smoothed Importance Sampling}
The solution to the first question is the Pareto smoothed importance sampling (PSIS). We give a brief review, and more details can be found in \citet{vehtari2015pareto}.

A generalized Pareto distribution with shape parameter $k$ and location-scale parameter $(\mu, \tau)$ has the density
$$
p(y|\mu, \sigma, k)=\left\{
\begin{aligned}
& \frac{1}{\sigma} \left( 1+k\left(  \frac{y-\mu}{\sigma} \right) \right) ^{-\frac{1}{k}-1}, & k \neq 0. \\
& \frac{1}{\sigma} \exp\left(  \frac{y-\mu}{\sigma} \right),   &  k = 0.\\
\end{aligned}
\right.
$$
PSIS  stabilizes importance ratios by fitting  a generalized Pareto distribution using the largest $M$ samples of $r_i$, where $M$ is empirically set as $\min(S/5, 3\sqrt{S})$. It then reports the estimated shape parameter $\hat k$ and replaces the $M$ largest $r_s$ by their expected value under the fitted generalized Pareto distribution. The other importance weights remain unchanged. We further truncate  all weights at the raw weight maximum $\max(r_s)$. The resulted smoothed weights are denoted by $w_s$, based on which a lower variance estimation can be calculated through (\ref{eqn:is}). 

Pareto smoothed importance sampling can be considered as Bayesian version of importance sampling with prior on the largest importance ratios. It has smaller mean square errors than plain IS and truncated-IS \citep{ionides2008truncated}.

\subsection{Using PSIS as a Diagnostic Tool}\label{sec:PSIS}
The fitted shape parameter $\doublehat k$, turns out to provide the desired diagnostic measurement between  the true posterior $p(\theta|y)$ and the VI approximation $q(\theta)$.   A  generalized Pareto distribution with shape $k$ has finite moments up to order $1/k$, thus any positive $\doublehat k$ value can be viewed as an estimate to
\begin{align}\label{eqn:k_hat}
 k=\inf\left\{  k'>0 : E_{q}\left( \frac{ p(\theta| y)}{  q(\theta)}  \right)^{\frac{1}{k'}} <\infty   \right\}.
\end{align} 
$\hat k$ is invariant under any constant multiplication of  $p$ or $q$, which explains why we can suppress the marginal likelihood (normalizing constant) $p(y)$ and replace the intractable $p(\theta|y)$ with $p(\theta,y )$ in \eqref{ratios}.

After log transformation, \eqref{eqn:k_hat} can be interpreted as  \emph{R\'enyi divergence} \citep{renyi1961measures}  with order $\alpha$ between  $p(\theta|y)$ and  $q(\theta)$:
\begin{align*}  
& k=\inf\left\{ k'>0 : D_{\frac{1}{k'}} \left(  p || q  \right)<\infty \right\},  \\
\mathrm{where}\: &D_{\alpha} \left(  p || q  \right) = \frac{1}{\alpha-1} \log \int_{\Theta} p(\theta)^{\alpha} q(\theta)^{1-\alpha}  d \theta. 
\end{align*}
It is well-defined since  {R\'enyi divergence} is monotonic increasing on order $\alpha$.  Particularly, when $ k >0.5$,   the  \emph{$\chi^2$ divergence}  $\chi (p || q)$, becomes infinite, and when $ k >1$, $D_{1}(p || q)=\KL(p,q) = \infty$, indicating a disastrous VI approximation, despite the fact that $\KL(q,p)$ is always minimized among the variational family. The connection to R\'enyi divergence holds when $k>0$. When $k<0$, it predicts the importance ratios are bounded from above.

This also illustrates the advantage of a continuous $\hat k$ estimate in our approach over only testing the existence of  second moment of $E_q(q/p)^2$ \citep{ epifani2008case, koopman2009testing} -- it indicates if the  R\'enyi divergence between $q$ and $p$ is finite for all continuous  order  $\alpha>0$.

Meanwhile, the shape parameter $ k$  determines the finite sample convergence rate of both IS and  PSIS adjusted estimate.    \citet{geweke1989bayesian} shows when  $\E_q[ r (\theta)^2 ]  < \infty$ and $\E_q[ \bigl( r (\theta)h(\theta) \bigr) ^2]< \infty $  hold (both conditions can be tested by  $ \hat k$ in our approach),  the central limit theorem guarantees the square root convergence rate. Furthermore, when $ k< 1/3 $, then the Berry-Essen theorem  states faster convergence rate to normality \citep{chen2004normal}.
\citet{cortes2010learning} and \citet{cortes2013relative} also link  the finite sample convergence rate of IS with   the number of existing moments of importance ratios. 

PSIS has smaller estimation error than the plain VI estimate, which we will experimentally verify this in Section \ref{sec:app}. A large $\hat k$ indicates the failure of finite sample PSIS, so it further indicates the large estimation error of VI approximation.  Therefore, even when the researchers' primary goal is not to use variational approximation $q$ as  an PSIS proposal, they should be alert by a large $\hat k$ which tells the  discrepancy between the VI approximation result and the true posterior.

According to empirical study in \citet{vehtari2015pareto}, we set the  threshold of  $\hat k$ as follows.
\begin{itemize}
	\item If $\doublehat k<0.5$, we can invoke the  central limit theorem to suggest PSIS has a fast convergence rate. We conclude the variational approximation $q$ is close enough to the  true density. We recommend further using PSIS to adjust the estimator \eqref{eqn:is} and calculate other divergence measures.
	\item If $0.5<\doublehat k<0.7$,  we still observe practically useful finite sample convergence rates and acceptable Monte Carlo error for PSIS. It indicates  the variational approximation $q$ is not perfect but still useful.  Again, we recommend PSIS to shrink errors.
	\item If $\doublehat k>0.7$, the PSIS convergence rate becomes impractically slow, leading to a large mean square error, and a even larger error for plain VI estimate.  We should consider  tuning the variational methods (e.g., re-parametrization, increase iteration times, increase mini-batch size, decrease learning rate, et.al.,) or turning to exact MCMC.  Theoretically   $ k $ is always smaller than 1, for $E_{q}\left[ { p(\theta|y)}/{  q(\theta)}  \right]=p(y) <\infty$, while in practice  finite sample estimate $\doublehat k$ may be larger than 1, which indicates even worse finite sample performance.
\end{itemize}

The  proposed diagnostic method is summarized in Algorithm \ref{alg:psis}.

\begin{algorithm}[tb]
	\caption{\em PSIS diagnostic}\label{alg:psis}
	\begin{algorithmic}[1]
		\STATE {\bfseries Input:} the  joint density function $p(\theta, y)$;  	number of posterior samples $S$;   number of tail samples $M$.
		\STATE Run variational inference to $p(\theta|y)$, obtain VI approximation $q(\theta)$;
		 
		\STATE 	Sample  $(\theta_s, s=1, \dots, S)$  from  $q(\theta)$;
		\STATE 	Calculate the importance ratio $r_{s}= p (\theta_s, y) /  q(\theta_s)$;
	 
		\STATE 	Fit  generalized Pareto distribution to the $M$ largest $ r_{s}$;
		\STATE 	Report the shape parameter $\doublehat k$;
		\IF{   $\hat k    < 0.7$  }
		\STATE	  Conclude VI approximation $q(\theta)$ is close enough to the unknown truth $p(\theta|y)$;
		\STATE	  Recommend further shrinking errors by PSIS.
		\ELSE
		\STATE	Warn users that the VI approximation is not  reliable.
		\ENDIF
	\end{algorithmic}
\end{algorithm}

\subsection{Invariance Under Re-Parametrization}
Re-parametrization is common in variational inference. Particularly, the \emph{reparameterization trick} \citep{rezende2014stochastic}  rewrites the objective function to make  gradient calculation easier in Monte Carlo integrations.  

A nice property of  PSIS diagnostics is that the $\hat k$ quantity is invariant under any re-parametrization.  Suppose $\xi$ = $T(\theta)$ is a smooth transformation, then the density ratio of $\xi$ under the target $p$ and the proposal $q$ does not change: 
$$ \frac{p(\xi)}{q(\xi)}   =  \frac{p\left( T^{-1}(\xi) \right)   | \mathrm{det} J_{\xi} T^{-1}(\xi)|  }  { q\left( T^{-1}(\xi) \right)   | \mathrm{det} J_{\xi} T^{-1}(\xi)| } =  \frac{p\left( \theta\right) }{q(\theta)}   $$
Therefore, ${p(\xi)}/{q(\xi)}$  and $ {p(\theta)}/{q(\theta)} $  have the same distribution under $q$, making it free to choose any convenient parametrization form when calculating $\hat k$.  

However,  if the re-parametrization changes the approximation family, then it will change the computation result, and PSIS diagnostics will change accordingly.  Finding the optimal parametrization form, such that the  re-parametrized posterior distribution lives exactly in the approximation family
$$p(T(\xi))=p\left( T^{-1}(\xi) \right)   |J_{\xi} T^{-1}(\xi)| \in \mathcal{Q} , $$
can be as hard as finding the true posterior. The PSIS diagnostic can guide the choice of  re-parametrization by simply comparing the $\hat k$ quantities of any parametrization.   Section \ref{sec:school} provides a practical example.

\subsection{Marginal PSIS Diagnostics Do Not Work}
As dimension increases, the VI posterior tends to be further away from the truth, due to the limitation of approximation families. As a result, $ k$ increases,  indicating inefficiency of importance sampling.  This is not the drawback of PSIS diagnostics. Indeed, when the focus is the joint distribution, such behaviour accurately reflects the quality of the variational approximation to the joint posterior.

Denoting the one-dimensional  true and approximate marginal density of   the $i$-th coordinate  $\theta_{i}$ as $p(\theta_{i}|y)$ and $q(\theta_{i})$, the  marginal ${k}$ for   $\theta_{i}$  can be defined as
$$  k_i=   \inf\left\{ 0< k' <1 :  \E_{q}\left(  \frac{p(\theta_i|y) }{q(\theta_i) } \right)^{\frac{1}{k'}}  < \infty  \right\}.$$
The marginal $ k_i$ is never larger (and usually smaller) than the joint ${k}$  in (\ref{eqn:k_hat}). 
\begin{proposition}\label{prop:marginal}
	For any two distributions $p$ and $q$ with support $\Theta$ and the margin index $i$,  if there is a number $\alpha>1$ satisfying $ E_q \left( p(\theta)/q(\theta)\right) ^ \alpha <\infty$, then $ E_q \left( p(\theta_i)/q(\theta_i)\right) ^ \alpha <\infty$.
\end{proposition}
Proposition \ref{prop:marginal} demonstrates why the importance sampling  is usually inefficient in high dimensional sample space, in that the joint estimation is ``worse'' than any of the marginal estimation.

Should we  extend the PSIS diagnostics to marginal distributions?   We find  two reasons why the marginal PSIS diagnostics can be misleading. Firstly, unlike the easy access to the unnormalized joint posterior distribution $p(\theta, y)$, the true marginal posterior density  $p(\theta_i|y)$ is typically unknown, otherwise one can conduct one-dimensional sampling easily to obtain the the marginal samples.  Secondly,  a smaller $\hat k_i$ does not necessary guarantee a well-performed marginal estimation. The marginal approximations in variational inference can both over-estimate and  under-estimate the tail thickness of one-dimensional distributions, the  latter situation gives rise to a smaller $\hat k_i$.  Section \ref{sec:school} gives an example, where the marginal approximations with extremely  small marginal ${k}$ have large estimation errors. This does not happen in the joint case as the direction of the Kullback-Leibler divergence $q^*(\theta)$ strongly penalizes too-heavy tails, which makes it unlikely that the tails of the variational posterior are significantly heavier than the tails of the true posterior.

\section{Assessing the Average Performance of the Point Estimate}
The proposed PSIS diagnostic assesses the quality of the VI approximation to the full posterior distribution. It is often observed that while the VI posterior may be a poor approximation to the full posterior, point estimates that are derived from it may still have good statistical properties. In this section, we propose a new method for assessing the calibration of the center of a VI posterior. 

\subsection{The Variational Simulation-Based Calibration (VSBC) Diagnostic} 

This diagnostic is based on the proposal of \citet{cook2006validation} for validating general statistical software.  They noted that if $\theta^{(0)} \sim p(\theta)$ and  $y\sim p( y \mid \theta^{(0)} )$, then $${\Pr}_{(y,\theta^{(0)} )}\left({\Pr}_{\theta \mid y}(\theta <\theta^{(0)} ) \leq \cdot)\right) = \text{Unif}_{[0,1]}([0,\cdot]).$$ 

To use the observation of \citet{cook2006validation} to assess the performance of a VI point estimate, we propose the following procedure. Simulate $M>1$ data sets $\{{y}_j\}_{j=1}^M$ as follows: Simulate $\theta_{j}^{(0)} \sim p(\theta)$ and then simulate ${y}_{(j)} \sim p(y \mid \theta_{j}^{(0)} )$, where   ${y}_{(j)}$ has the same dimension as $y$.  For each of these data sets, construct a variational approximation to $p(\theta \mid y_{j}) $ and compute the marginal calibration probabilities $p_{ij} = \Pr_{\theta \mid y_{(j)}}\left( \theta_i \leq [\theta_{j}^{(0)} ]_i\right)$.

To apply the full procedure of \citet{cook2006validation}, we would need to test $\operatorname{dim}(\theta)$ histograms for uniformity, however this would be too stringent a check as, like our PSIS diagnostic, this test is only passed if the variational posterior is a good approximation to the true posterior.   Instead, we follow an observation of \citet{anderson1996method} from the probabilistic forecasting validation literature and note that  asymmetry in the histogram for $p_{i:}$ indicates bias in the variational approximation to the marginal posterior $\theta_i \mid y$.

The VSBC diagnostic tests for symmetry of the marginal calibration probabilities around $0.5$ and  either by visual inspection of the histogram or by using a Kolmogorov-Smirnov (KS) test to evaluate whether $p_{i:}$ and $1-p_{i:}$ have the same distribution.  When $\theta$ is a high-dimensional parameter, it is important to interpret the results of any hypothesis tests through a multiple testing lens.

\begin{algorithm}[tb]
	\caption{\em VSBC marginal diagnostics}\label{alg:VSBC}
	\begin{algorithmic}[1]
		\STATE {\bfseries Input:} prior density $p(\theta)$, data likelihood $p(y 
		\mid \theta)$;     number of replications $M$;  parameter dimensions $K$; 
		\FOR {$j=1:M$}
		\STATE Generate $\theta_{j}^{(0)}$ from prior $p(\theta)$;
	
		\STATE Generate a size-$n$ dataset $\left( y_{(j)}\right)  $ from $p(y \mid \theta_{j}^{(0)})$;
		\STATE Run variational inference using dataset $y_{(j)} $, obtain a VI approximation distribution $q_j(\cdot)$
			\FOR {$i=1:K$}
			 \STATE Label ${\theta}_{ij}^{(0)} $ as the $i$-th marginal component of $\theta_{j}^{(0)}$;  Label ${\theta}_{i}^* $ as the $i$-th marginal component of $\theta^*$;
			 \STATE  Calculate $p_{ij}=\mathrm{Pr}(\theta_{ij}^{(0)}<\theta^*_i \mid \theta^*\sim q_j  )$
		\ENDFOR
\ENDFOR
		
	\FOR {$i=1:K$}
\STATE  Test if the distribution of $\left\{ p_{ij}\right\}_{j=1}^{M} $ is symmetric;
\STATE    If rejected,  the VI approximation is biased in its $i$-th margin.
\ENDFOR
\end{algorithmic}
\end{algorithm}

\subsection{Understanding the VSBC Diagnostic}

Unlike the PSIS diagnostic, which focuses on a the performance of variational inference for a fixed data set $y$, the VSBC diagnostic assesses the \emph{average} calibration of the point estimation over all datasets that could be constructed from the model. Hence, the VSBC diagnostic operates under a different paradigm to the PSIS diagnostic and we recommend using both as appropriate. 

There are two disadvantages to this type of calibration when compared to the PSIS diagnostic.  As is always the case when interpreting hypothesis tests, just because something works on average doesn't mean it will work for a particular realization of the data. The second disadvantage is that this diagnostic does not cover the case where the observed data is not well represented by the model.  We suggest interpreting the diagnostic conservatively: if a variational inference scheme fails the diagnostic, then it will not perform well on the model in question.  If the VI scheme passes the diagnostic, it is not guaranteed that it will perform well for real data, although if the model is well specified it should do well.

The VSBC diagnostic has some advantages compared to the PSIS diagnostic. It is well understood that, for complex models, the VI posterior can be used to produce a good point estimate even when it is far from the true posterior. In this case, the PSIS diagnostic will most likely indicate failure.  The second advantage is that unlike the PSIS diagnostic, the VSBC diagnostic considers one-dimensional marginals $\theta_i$ (or any functional $h(\theta)$), which allows for a more targeted interrogation of the fitting procedure.

With stronger assumptions, The VSBC test can be formalized as   in Proposition \ref{prop:VSBC}.

\begin{proposition}\label{prop:VSBC}
	Denote $\theta$  as  a one-dimensional parameter  that is of interest.
	Suppose in addition we have: 
	(i) the VI approximation  $q$ is symmetric;
	(ii) the true posterior $p(\theta |y) $ is symmetric.
	If the VI estimation $q$ is unbiased, i.e., $\E_{\theta\sim q(\theta|y)}\theta =  \E_{\theta \sim p(\theta|y)} \theta,$
	then the  distribution of  VSBC $p$-value is symmetric. 
	Otherwise, if the VI estimation  is positively/negatively biased, 
	then the  distribution of  VSBC $p$-value is right/left skewed. 
\end{proposition}
The symmetry of the true posterior is a stronger assumption than is needed in practice for this result to hold.  In the forecast evaluation literature, as well as the literature on posterior predictive checks, the symmetry of the histogram is a commonly used heuristic to assess the potential bias of the distribution. In our tests, we have seen the same thing occurs: the median of the variational posterior is close to the median of the true posterior when the VSBC histogram is symmetric. We suggest again that this test be interpreted conservatively: if the histogram is not symmetric, then the VI is unlikely to have produced a point estimate close to the median of the true posterior. 
\section{Applications}\label{sec:app}
Both PSIS and VSBC diagnostics are applicable to any variational inference algorithm. Without loss of generality, we implement  mean-field Gaussian automatic differentiation variational inference (ADVI) in this section.

\subsection{Linear Regression}
Consider a Bayesian linear regression $y\sim \N(X\beta, \sigma^2)$ with prior $\{\beta_i\}_{i=1}^{K}\sim \N(0,1), \sigma\sim \mathrm{gamma}(.5,.5)$. We  fix sample size $n=10000$ and number of regressors $K=100$.   

Figure \ref{fig:VSBC_linear} visualizes the VSBC diagnostic, showing the distribution of VSBC $p$-values of the first  two regression coefficients $\beta_1,$ $\beta_2$ and  $\log \sigma$ based on $M=1000$ replications.  The two sided Kolmogorov-Smirnov test for $p_{:}$ and $1-p_{:}$ is only rejected for $p_{\sigma:}$, suggesting the VI approximation is in average marginally unbiased for $\beta_1$ and $\beta_2$, while $\sigma$ is over-estimated as $p_\sigma$ is right-skewed.  The under-estimation of  posterior variance is  reflected by the U-shaped distributions.

Using one randomly generated dataset in the same problem, the PSIS $\doublehat k$ is  $0.61$, indicating the  joint  approximation  is  close to the true  posterior. However, the performance of ADVI is sensitive to the stopping time,  as in any other optimization problems.  As displayed in the left panel of Figure \ref{fig:linear_large_n}, changing the  threshold  of relative ELBO change from a conservative  $10^{-5}$ to the default recommendation $10^{-2}$ increases $\doublehat k$ to $4.4$, even though   $10^{-2}$ works fine for many other simpler problems. In this example, we can also view $\doublehat k$ as a convergence test. The right panel shows  $\doublehat k$ diagnoses estimation error, which eventually become negligible in PSIS adjustment when $\doublehat k<0.7$.   To account for the uncertainty  of stochastic optimization  and $\hat k$ estimation, simulations are repeated 100 times. 

\begin{figure}
	\centerline{\includegraphics[width=\columnwidth]{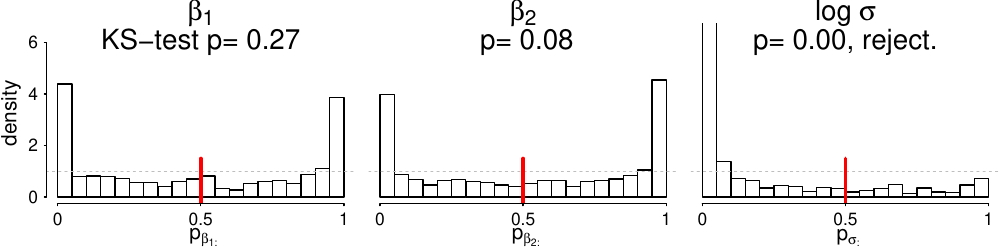}}
	\vskip -0.1in
	\caption{\em VSBC diagnostics for $\beta_1, \beta_2$ and $\log \sigma$  in the Bayesian linear regression example. The VI estimation overestimates $\sigma$  as $p_\sigma$ is right-skewed, while $\beta_1$ and $\beta_2$ is unbiased as the two-sided KS-test is not rejected.} \label{fig:VSBC_linear} \vskip -0.05in
\end{figure}

\begin{figure}
	\centerline{\includegraphics[width=\columnwidth]{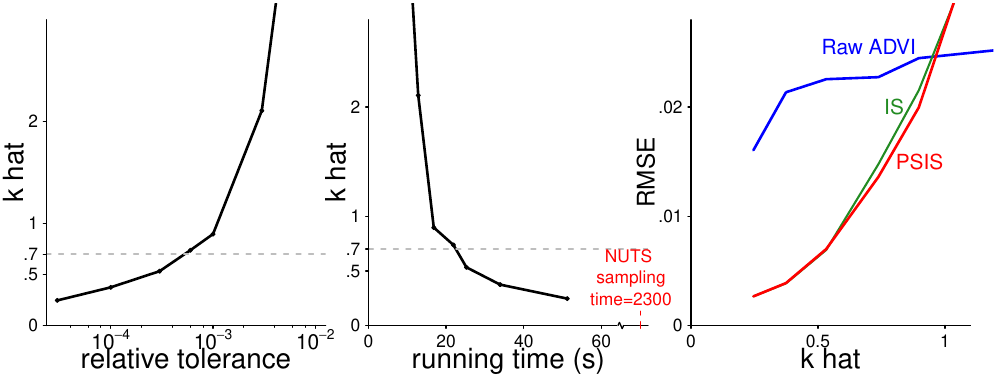}}
	\vskip -0.1in
	\caption{\em  ADVI is sensitive to the stopping time in the linear regression example.  The default 0.01 threshold lead to a fake convergence, which can be diagnosed by monitoring  PSIS $\hat k$. PSIS adjustment always shrinks the estimation errors.} \label{fig:linear_large_n}\vskip -0.05in
\end{figure}

\subsection{Logistic Regression}

Next we  run ADVI to a logistic regression $Y \sim \mathrm{Bernoulli} \left( \mathrm{logit}^{-1}( \beta X  ) \right) $ with a flat prior on $\beta$.  We  generate  $X=(x_1,\dots, x_n)$ from  $\hbox{N}(0, (1-\rho) I_{K\times K}+\rho 1_{K\times K})$  such that the correlation in design matrix is $\rho$, and $\rho$ is changed from 0 to 0.99.   The first panel in Figure \ref{fig:logistic_lpd} shows PSIS $\hat k$ increases as the design matrix correlation increases. It is not monotonic because $\beta$ is initially negatively correlated when $X$ is independent. A large $\rho$ transforms into a large correlation for posterior distributions in $\beta$, making it harder to be approximated by a mean-field family, as can be diagnosed by $\hat k$.   In panel 2 we  calculate mean log predictive density (lpd) of  VI approximation and true posterior using 200 independent test sets.   Larger $\rho$ leads to worse mean-field approximation, while  prediction becomes easier. Consequently, monitoring lpd does not diagnose the VI behavior; it increases (misleadingly suggesting better fit) as $\rho$ increases.  In this special case, VI has larger lpd than the true posterior, due to the VI under-dispersion and the model misspecification. Indeed, if viewing  lpd as a function $h(\beta)$, it is the discrepancy between VI lpd and true lpd that reveals the  VI performance, which can also be diagnosed by $\hat k$. Panel 3 shows a sharp increase of lpd discrepancy around $\hat k=0.7$, consistent with the empirical threshold we suggest.

\begin{figure}
	\begin{center}
		\includegraphics[width=\columnwidth]{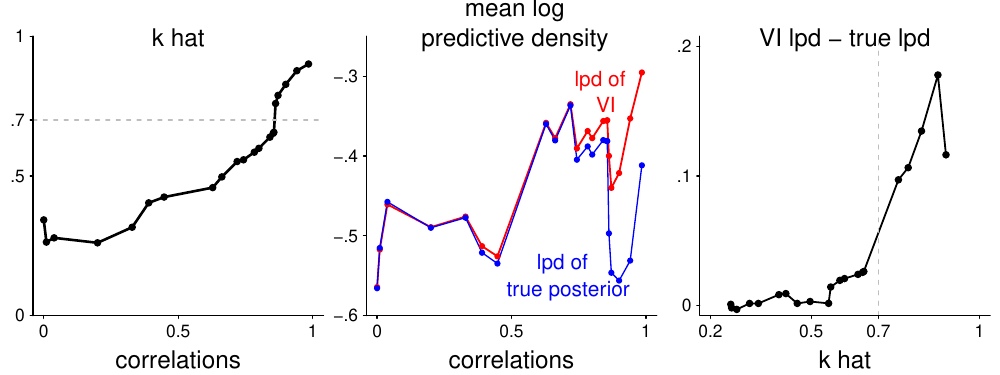}
	\end{center}\vspace{-0.15in}
	\caption{\em In the logistic regression example, as the correlation in design matrix increase, the  correlation in parameter space also increases, leading to larger $\hat k$. Such flaw is hard to tell from the VI log predictive density (lpd), as a larger correlation makes the prediction easier.  $\hat k$ diagnose the discrepancy of VI lpd and true posterior lpd, with a sharp jump at $0.7$.} \label{fig:logistic_lpd}  \vspace{-0.05in}
\end{figure}

Figure  \ref{fig:logistic_mse} compares the first  and second moment  root mean square errors (RMSE) $||E_p \beta -E_{q^*} \beta ||_2$ and $||E_p \beta^2 -E_{q^*} \beta^2 ||_2$ in the previous example using three estimates: (a) VI without post-adjustment, (b) VI adjusted by vanilla importance sampling, and  (c) VI adjusted by PSIS. 

\begin{figure}
	\begin{center}
		\includegraphics[width=\columnwidth]{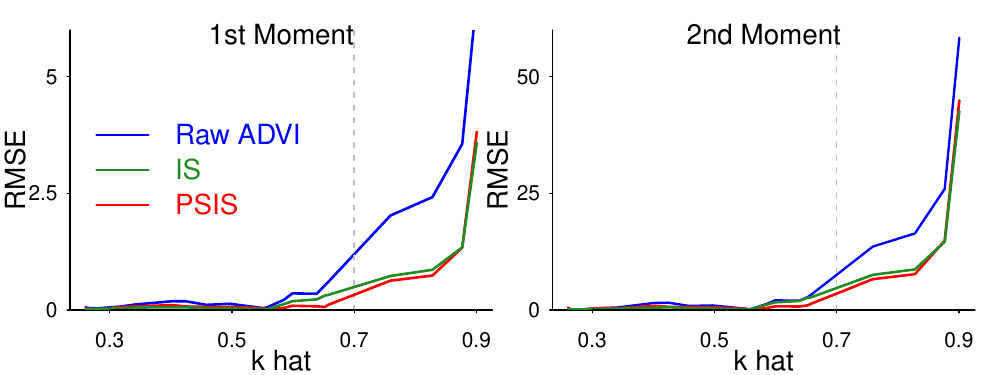}
	\end{center}\vspace{-0.15in}
	\caption{\em In the logistic regression with varying  correlations, the $\hat k$  diagnoses the  root mean square of first and second  moment  errors.  No estimation is reliable when $\hat k>0.7$. Meanwhile, PSIS adjustment always shrinks the VI estimation errors.    } \label{fig:logistic_mse} \vspace{-0.05in}
\end{figure}

PSIS diagnostic  accomplishes two tasks here: (1) A small $\hat k$ indicates that VI approximation is  reliable. When $\hat k>0.7$, all estimations are no longer reasonable so the user should be alerted. (2) It further improves the approximation using PSIS adjustment,  leading to a quicker convergence rate and  smaller mean square errors  for both  first and second moment estimation.  Plain importance sampling has larger RMSE for it suffers from a larger variance.

\subsection{Re-parametrization in a  Hierarchical Model} \label{sec:school}

The \emph{Eight-School Model}  \citep[Section 5.5]{gelman2013bayesian} is the simplest Bayesian hierarchical normal model.  
Each school reported  the treatment effect  mean  $y_i$  and standard deviation $\sigma_i$  separately.   There was no prior reason to believe that any of the treatments were more effective than any other, so we model them as independent experiments:
\begin{align*}
	 y_j|\theta_j   & \sim \mbox{N} (\theta_j, \sigma_j^2), \quad  \theta_j |\mu, \tau  \sim \mbox{N} (\mu, \tau^2), \quad 1 \leq j \leq 8, \\
	 \mu &\sim \N(0,5),  \quad  \quad    \tau  \sim \mathrm{half\!\!-\!\!Cauchy}(0, 5).
\end{align*}
where $\theta_j$ represents the  treatment effect in school $j$, and $\mu$ and  $\tau$ are the hyper-parameters shared across all schools. 

In this hierarchical model, the conditional variance of $\theta$ is strongly dependent on the  standard deviation $\tau$, as shown by the joint sample of $\mu$ and  $\log \tau$ in the bottom-left corner in  Figure \ref{fig:8school_psis}.   The Gaussian assumption in ADVI  cannot capture such structure. More interestingly,  ADVI over-estimates the posterior variance for all parameters $\theta_1$ through $\theta_8$, as shown by positive biases of their posterior standard deviation in the last panel.  In fact, the posterior mode  is at $\tau=0$, while the entropy penalization keeps VI estimation  away from it,  leading to an overestimation due to the funnel-shape. Since the conditional expectation $ \E[\theta_i | \tau, y, \sigma ]= \left(  {\sigma_j^{-2}} +{\tau^{-2}}\right) ^{-1} $ is an increasing  function on $\tau$,  a positive bias of $\tau$ produces over-dispersion of $\theta$.

 \begin{figure}
 	\begin{center}
 		\includegraphics[width=\columnwidth]{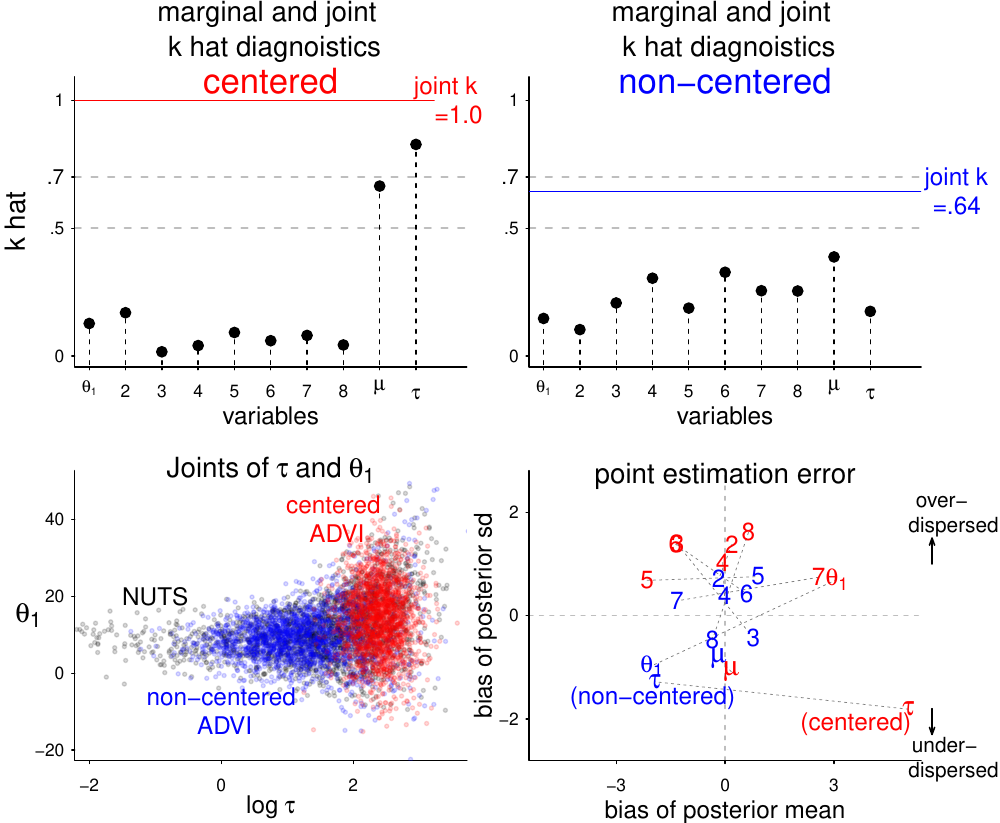}
 	\end{center}\vspace{-0.15in}
 	\caption{\em The upper two panels shows the joint and marginal PSIS diagnostics of the eight-school example. The centered parameterization has $\hat k > 0.7$, for it cannot capture the funnel-shaped dependency between $\tau$ and $\theta$. The bottom-right panel shows the bias of posterior mean and standard errors of marginal distributions. Positive bias of $\tau$ leads to over-dispersion of $\theta$. } \label{fig:8school_psis} \vspace{-0.05in}
 \end{figure}

 The top left panel shows the marginal and joint PSIS diagnostics.  The joint $\hat k$  is 1.00, much beyond the threshold, while  the marginal $\hat k$ calculated through the true marginal distribution for all $\theta$ are misleadingly small due to the over-dispersion.

Alerted by such large $\hat{k}$, researchers should seek some improvements, such as re-parametrization.  The \emph{non-centered parametrization}  extracts  the dependency between $\theta$ and $\tau$ through a transformation $	\theta^* = (\theta- \mu)/\tau$:
$$ y_j|\theta_j    \sim \mbox{N} (\mu+\tau	\theta^*_j, \sigma_j^2), \quad  \theta^*_j   \sim \mbox{N} (0, 1). $$
There is no general rule to determine whether non-centered parametrization is  better than the centered one and there are many other parametrization forms. Finding the optimal parametrization  can be as hard as finding the true posterior, but $\hat{k}$ diagnostics  always guide the choice of  parametrization.   As shown by the top right panel in Figure \ref{fig:8school_psis},
 the joint $\hat{k}$ for the  non-centered ADVI decreases to $0.64$ which indicated the approximation is not perfect but reasonable and usable. The bottom-right panel demonstrates that the re-parametrized ADVI posterior is much closer to the truth, and has  smaller biases for both first and second moment estimations.

 \begin{figure}
	\begin{center}
		\includegraphics[width=\columnwidth]{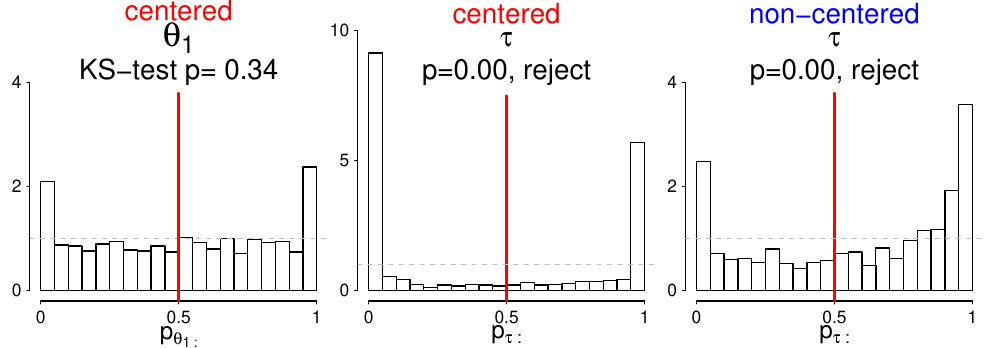}
	\end{center}\vspace{-0.15in}
	\caption{\em In the eight-school example, the VSBC diagnostic verifies VI estimation of $\theta_1$ is unbiased  as the distribution of $p_{\theta_{1:}}$ is symmetric. $\tau$ is overestimated in the centered parametrization and underestimated in the non-centered one, as told by the right/ left skewness of $p_{\tau:}$. } \label{fig:8school_VSBC} \vspace{-0.05in}
\end{figure}

We can  assess the marginal estimation using VSBC diagnostic, as summarized in Figure \ref{fig:8school_VSBC}. In the centered parametrization, the point estimation for $\theta_1$ is in average unbiased, as the two-sided KS-test is not rejected.   
The histogram for $\tau$ is right-skewed, for we can reject one-sided KS-test with the alternative to be $p_{\tau:}$ being stochastically smaller  than $p_{\tau:} $. Hence we conclude $\tau$ is over-estimated in the centered parameterization. On the contrast, the non-centered $\tau$ is negatively biased, as diagnosed by the left-skewness of $p_{\tau:}$. Such conclusion is consistent with the bottom-right panel in Figure \ref{fig:8school_psis}.

To sum up, this example illustrates how the Gaussian family assumption can be unrealistic even for a simple hierarchical model. It also clarifies VI posteriors can be both over-dispersed and under-dispersed, depending crucially on the true parameter dependencies. Nevertheless, the recommended PSIS and VSBC diagnostics provide a practical summary of the computation result.

\subsection{Cancer Classification Using  Horseshoe Priors} \label{HS}
We illustrate how the proposed diagnostic methods work in the Leukemia microarray cancer dataset that contains  $D=7129$ features and $n=72$ observations. Denote $y_{1:n}$ as binary outcome and $X_{n\times D}$ as the predictor, the logistic regression with a regularized horseshoe prior \citep{piironen2017sparsity} is given by
\begin{align*}
&y |\beta \sim \mathrm{Bernoulli}\left(  \mathrm{logit}^{-1}\left(  X\beta\right)  \right) , \quad \beta_j|\tau, \lambda, c \sim \N(0, \tau^2 \tilde \lambda_j^2) ,\\
     & \lambda_j \sim \mathrm{C}^{+} (0,1),   \quad   \tau   \sim \mathrm{C}^{+} (0,\tau_0), \quad c^2 \sim \mathrm{Inv\!\!-\!\!Gamma} (2,8).
\end{align*}
where $\tau>0$ and $\lambda>0$ are global and local shrinkage parameters, and $\tilde \lambda_j^2   =  c^2 \lambda_j^2 /\left( c^2 + \tau^2 \lambda_j^2\right) $.  The regularized horseshoe prior adapts to the sparsity   and allows us to specify a minimum level of regularization to the largest values.

ADVI is computationally appealing for it only takes a few minutes while MCMC sampling takes hours on this dataset. However,  PSIS diagnostic gives $\doublehat k=9.8$ for ADVI, suggesting the VI approximation is not even close to the true posterior. Figure \ref{fig:horseshoe} compares the ADVI and true posterior density of $\beta_{1834}$, $\log \lambda_{1834}$ and $\tau$. The Gaussian assumption makes it impossible to recover the bimodal distribution of some $\beta$. 

 \begin{figure}[ht]
	\begin{center}
		\includegraphics[width=\columnwidth]{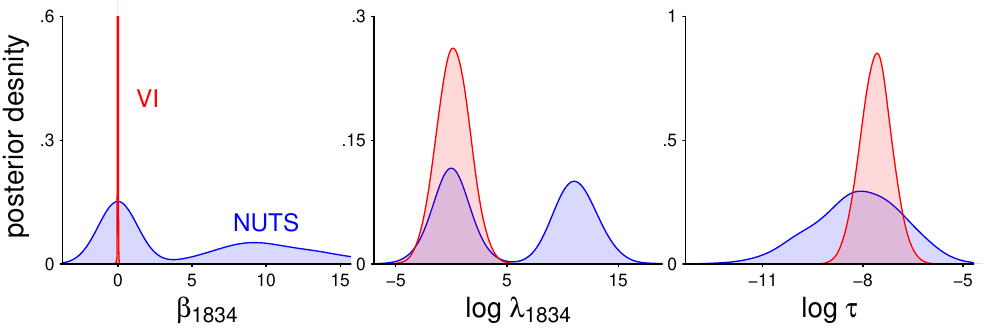}
	\end{center}\vspace{-0.05in}
	\caption{\em The comparison of ADVI and true posterior density of $\theta_{1834}$, $\log \lambda_{1834}$ and $\tau$ in the horseshoe logistic regression. ADVI misses the right mode of $\log \lambda$, making $\beta \propto \lambda$ become a spike. } \label{fig:horseshoe} \vspace{-0.08in}
\end{figure}

 \begin{figure}[ht]
	\begin{center}
		\includegraphics[width=\columnwidth]{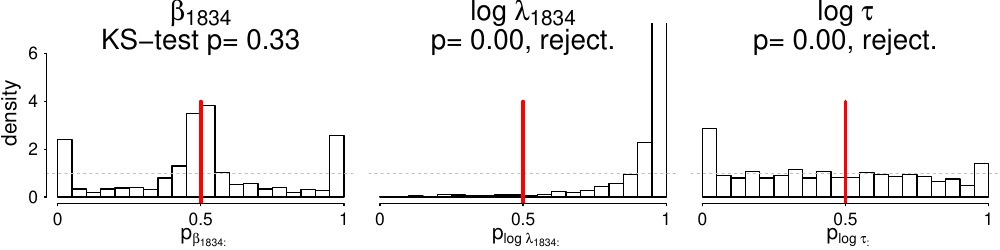}
	\end{center}\vspace{-0.05in}
	\caption{\em VSBC test in the horseshoe logistic regression. It tells the positive bias of $\tau$  and negative bias of $\lambda_{1834}$. $\beta_{1834}$ is in average unbiased for its  symmetric prior.} \label{fig:horseshoe_VSBC} \vspace{-0.08in}
\end{figure}

The VSBC diagnostics  as shown in Figure \ref{fig:horseshoe_VSBC} tell  the negative bias of local shrinkage $\lambda_{1834}$ from the left-skewness of $p_{\log \lambda_{1834}}$, which is the consequence of the right-missing mode.  For compensation, the global shrinkage $\tau$ is over-estimated, which is in agreement with the right-skewness of $p_{\log \tau}$.   $\beta_{1834}$ is in average unbiased, even though it is strongly underestimated from  in Figure \ref{fig:horseshoe}.  This is because VI estimation is mostly a spike at 0 and its prior is symmetric.   As we have explained, passing the VSBC test means the average unbiasedness,  and does not ensure the unbiasedness for a specific parameter setting. This is the price that  VSBC pays for averaging over all priors.

\section{Discussion}

\subsection{The Proposed Diagnostics are Local}
As no single diagnostic method can tell all problems,  the proposed diagnostic methods have  limitations.  The PSIS diagnostic is limited when the posterior is multimodal as  the samples drawn from  $q(\theta)$ may not cover all the modes of the posterior and the estimation of $ k$ will  be indifferent to the unseen modes. In this sense, the PSIS diagnostic is a \emph{local diagnostic} that will not  detect unseen modes.  For example,  imagine the true posterior is $p= 0.8 \mbox{N}(0, 0.2) +0.2 \mbox{N}(3, 0.2)$ with two isolated modes.  Gaussian family VI  will converge to one of the modes, with the importance ratio to be  a constant number $0.8$ or $0.2$.  Therefore  ${k}$ is 0, failing to penalize the missing density.   In fact, any divergence measure based on samples from the approximation such as $\KL(q, p)$ is \emph{local}. 

 The bi-modality can be detected by multiple over-dispersed initialization.   It can also be diagnosed by other divergence measures such as $\KL(p,q)=  \E_{p} \log(q/p) $, which is computable through PSIS by letting $h= \log(q/p)$.
 
 In practice a marginal missing mode will typically lead to large joint discrepancy that is still detectable by  $\hat k$, such as in Section \ref{HS}. 
  
The  VSBC test, however, samples the true parameter from the prior distribution directly. Unless the prior is too restrictive, the VSBC $p$-value will diagnose the potential missing mode. 
 

\subsection {Tailoring Variational Inference for Importance Sampling  }
The PSIS diagnostic makes use of stabilized IS to diagnose VI. By contrast, can we modify VI to give a better IS proposal?
 
\citet{geweke1989bayesian} introduce an optimal proposal distribution based on split-normal and split-$t$,  implicitly minimizing the $\chi^2$ divergence between $q$ and $p$. Following this idea, we could first find the usual VI solution, and then switch Gaussian to Student-$t$ with a scale chosen to minimize the $\chi^2$ divergence.  

More recently, some progress is made to carry out  variational inference based on  R\'enyi  divergence \citep{li2016renyi, dieng2017variational}. But a big $\alpha$, say $\alpha=2$, is only meaningful  when the proposal has a much heavier tail than the target. For example,  a normal family does not contain any member having finite $\chi^2$ divergence to a Student-$t$ distribution, leaving the optimal objective function defined by \citet{dieng2017variational}  infinitely large.  
 
 There are several research directions. First, our proposed diagnostics are applicable to these modified approximation methods.
 Second,  PSIS re-weighting will give a more reliable importance ratio estimation  in the R\'enyi  divergence variational inference. 
 Third,  a continuous $\hat k$ and the corresponding  $\alpha$ are more desirable than only fixing $\alpha=2$, as the latter one does not necessarily have a finite result. Considering the role $\hat k$ plays in the importance sampling, we can optimize the discrepancy $D_\alpha (q || p)$ and   $\alpha>0$ simultaneously. We  leave this for future research.

 \clearpage
\section*{Acknowledgements}
The authors acknowledge support from the Office of Naval Research grants N00014-15-1-2541 and N00014-16-P-2039, 
the National Science Foundation grant CNS-1730414, and the Academy of Finland grant 313122.

\bibliography{diagnostics}
\bibliographystyle{icml2018}
\appendix
\newpage

\onecolumn
\icmltitle{Supplement to ``Yes, but Did It Work?: Evaluating Variational Inference"  }
\icmltitlerunning{Supplement to ``Evaluating Variational Inference"}

   \renewcommand{\thefigure}{\Roman{figure}} 
\setcounter{proposition}{0}
 \section {Sketch of Proofs}
 \subsection{ Proof to Proposition 1: Marginal $\hat k$ in PSIS diagnostic}
 \begin{proposition}
 	For any two distributions $p$ and $q$ with support $\Theta$ and the margin index $i$,  if there is a number $\alpha>1$ satisfying $ E_q \left( p(\theta)/q(\theta)\right) ^ \alpha <\infty$, then $ E_q \left( p(\theta_i)/q(\theta_i)\right) ^ \alpha <\infty$.
 \end{proposition}
 
\begin{proof}
		Without loss of generality, we could assume $\Theta=\R^K$, otherwise a smooth transformation is conducted. 
		
For any $1\leq i \leq K$, $p(\theta_{-i} | \theta_{i} )$ and $q(\theta_{-i} | \theta_{i} )$ define the conditional distribution of $\left( \theta_1,\dots,\theta_{i-1}, \theta_{i+1},\dots,\theta_K \right)  \in \R^{K-1}$ given $\theta_i$ under the  true posterior $p$ and the approximation $q$ separately.
	
	For any given index $\alpha>1$, Jensen inequality yields 
	\begin{align*}
	\int_{\R^{K-1}} \left( \frac{p(\theta_{-i} | \theta_{i} )}{q(\theta_{-i} | \theta_{i} )}\right)^\alpha q(\theta_{-i} | \theta_{i} ) \geq  \left( \int_{\R^{K-1}}  \frac{p(\theta_{-i} | \theta_{i} )}{q(\theta_{-i} | \theta_{i} )} q(\theta_{-i} | \theta_{i} )  \right)^\alpha =1
	\end{align*}
	Hence  \begin{align*}
	\int_{\R^{K}}\left(  \frac{ p (\theta)}{q (\theta)}\right) ^\alpha q(\theta) d\theta  &=   \int_{\R^{K-1}}  \int_{\R} \left(  \frac{ p (\theta_i) p (\theta_{-i} | \theta_i)   }{q (\theta_i)  q (\theta_{-i} | \theta_i) } \right) ^\alpha q (\theta_i)  q (\theta_{-i} | \theta_i) d \theta_i d\theta_{-i} \\
	&=   \int_{\R}  \left(  \int_{\R^{K-1}}   \left(  \frac{ p (\theta_{-i} | \theta_i)   }{ q (\theta_{-i} | \theta_i) } \right) ^\alpha   q (\theta_{-i} | \theta_i)   d\theta_{-i}   \right)    \left(  \frac{p (\theta_i) }{q(\theta_i)} \right) ^\alpha  q(\theta_i)    d \theta_i  \\
	&\geq   \int_{\R}      \left(  \frac{p (\theta_i) }{q(\theta_i)} \right) ^\alpha  q(\theta_i)    d \theta_i 
	\end{align*}
\end{proof}

 \subsection{Proof to Proposition 2: Symmetry in VSBC-Test }
 \begin{proposition}
 For a one-dimensional  parameter  $\theta$   that is of interest,  Suppose in addition we have: \\
 (i) the VI approximation  $q$ is symmetric;\\
 (ii) the true posterior $p(\theta |y) $ is symmetric.\\
 If the VI estimation $q$ is unbiased, i.e., $$   \E_{\theta\sim q(\theta|y)}\theta =  \E_{\theta \sim p(\theta|y)} \theta, \forall y$$ 
 Then the  distribution of  VSBC $p$-value is symmetric. \\
If the VI estimation  is positively/negatively biased, 
  then the  distribution of  VSBC $p$-value is right/left skewed. 
 \end{proposition}
In the proposition we write $q(\theta|y)$ to emphasize that the VI approximation also depends on the observed data.
 \begin{proof}
 	First, as the same logic in \cite{cook2006validation}, when $\theta^{(0)}$ is sampled from its prior  $p(\theta)$ and simulated data $y$ sampled from likelihood $p(y|\theta^{(0)})$, $(y,\theta^{(0)} )$ represents a sample from the joint distribution $p(y, \theta)$  and therefore $\theta^{(0)}$ can be viewed as a draw from $p(\theta|y)$, the true posterior distribution of $\theta$ with $y$ being observed.
 	
 	 We denote $q(\theta^{(0)})$ as the VSBC $p$-value of the sample $\theta^{(0)}$.  Also denote $Q_x(f)$ as the $x-$quantile ($x \in [0,1]$) of any distribution $f$. 
 	To prove the result, we need to show 
 	$$ 1- \Pr(q(\theta^{(0)}) < x) = \Pr(q(\theta^{(0)}) < 1-x), \forall x \in [0,1], $$
 \begin{align*}
\mathrm{LHS}&=  \Pr\left( q(\theta^{(0)}) > x\right) \\
&=  \Pr\left(   \theta^{(0)}   > Q_x \left( q(\theta| y) \right)  \right).
\end{align*}
 \begin{align*}
\mathrm{RHS}= \Pr\left( \theta^{(0)} < Q_{1-x} \left( q(\theta| y) \right) \right)  &= \Pr\left( \theta^{(0)} < 2 E_{q(\theta|y)}\theta - Q_{x} \left( q(\theta| y) \right) \right)  \\
&= \Pr\left( \theta^{(0)} < 2 E_{p(\theta|y)}\theta - Q_{x} \left( q(\theta| y) \right) \right)  \\
&= \Pr\left( \theta^{(0)} > Q_{x} \left( q(\theta| y) \right) \right)  \\
&= \mathrm{LHS}
\end{align*}
The first equation above	uses the symmetry of $q(\theta|y)$, the second equation comes from the  the unbiasedness condition.  The third 	is the result of the symmetry of $p(\theta|y)$.

If the VI estimation  is positively biased, $ \E_{\theta\sim q(\theta|y)}\theta >  \E_{\theta \sim p(\theta|y)} \theta, \forall y,$ then we change the second 
equality sign into a  less-than sign. 
 \end{proof}

 \section {Details of Simulation Examples}
In this section, we give more detailed description of the simulation examples in the  manuscript.  We use Stan \citep{stan2017stan} to implement both automatic differentiation variational inference (ADVI) and Markov chain Monte Carlo (MCMC) sampling. We implement Pareto smoothing through R package ``loo'' \citep{loopackage}.
We also provide all the source code in \url{https://github.com/yao-yl/Evaluating-Variational-Inference}.
 
 \subsection{Linear and Logistic Regressions}
In Section 4.1, We start with a Bayesian linear regression $y\sim \N(X\beta, \sigma^2)$ without intercept. The prior is set as  $\{\beta_i\}_{i=1}^{d}\sim \N(0,1), \sigma\sim \mathrm{gamma}(0.5,0.5)$. We  fix sample size $n=10000$ and number of regressors $d=100$.    Figure \ref{ls_linear_reg} displays the Stan code.
   \begin{figure}[htb]
 \begin{lstlisting}[language=C++ ]
 data {
 int <lower=0> n;     //number of observations, we fix n=10000 in the simulation;
 int <lower=0> d;      //number of predictor variables,  fix d=100;
 matrix [n,d] x ;       // predictors;
 vector [n] y;         // outcome;
 }
 parameters {
 vector [d] b;    // linear regression coefficient;
 real <lower=0> sigma;    //linear regression std;
 }
 model {
 y ~ normal(x * b, sigma);  
 b ~ normal(0,1);  // prior for  regression coefficient;
 sigma ~ gamma(0.5,0.5);    // prior for  regression std.
 }
 \end{lstlisting}\caption{\em Stan code for linear regressions}\label{ls_linear_reg}
 \end{figure}

 We find ADVI can be sensitive to  the stopping time. Part of the reason is  the objective function itself is evaluated through Monte Carlo samples, producing  large uncertainty.  In the current version of Stan, ADVI  computes the running average and running median of the relative ELBO norm changes. Should either number fall below a threshold \texttt{tol\_rel\_obj}, with the default value to be 0.01,   the algorithm is considered converged. 

In Figure 1 of the main paper, we run VSBC test on ADVI approximation. ADVI is deliberately tuned in a conservative way. The convergence tolerance is set as \texttt{tol\_rel\_obj}=$10^{-4}$  and the learning rate is $\eta =0.05$.  The predictor $X_{10^5\times10^2}$ is fixed in all replications and is generated independently from $\N(0,1)$.  To avoid multiple-comparison problem, we pre-register the first and second coefficients $\beta_1$ $\beta_2$ and $\log \sigma$ before the test. The VSBC diagnostic is based on $M=1000$ replications.

In Figure 2 we independently generate  each coordinate of $\beta$ from $\N(0,1)$ and set a relatively large variance $\sigma=2$.  The predictor $X$ is  generated independently from $\N(0,1)$ and $y$ is sampled from the normal likelihood. We vary the threshold \texttt{tol\_rel\_obj} from $0.01$ to $10^{-5}$ and show the trajectory of $\hat k$ diagnostics.  The $\hat k$ estimation, IS and  PSIS adjustment are all  calculated from $ S= 5 \times 10^4$ posterior samples. We ignore the ADVI  posterior sampling time. 
The actual  running time is based on a laptop experiment result (2.5 GHz processor, 8 cores).The exact sampling time is based on the No-U-Turn Sampler (NUTS, \citealt{hoffman2014no}) in Stan with 4 chains and 3000 iterations in each chain.  We also calculate the root mean square errors (RMSE) of all parameters $ ||E_p[ (\beta, \sigma)]- E_q [(\beta, \sigma) ]||_{L^{2}}$, where $(\beta, \sigma)$ represents the combined vector of all $\beta$ and $\sigma$. To account for the uncertainty,  $\hat k$, running time, and RMSE takes the average of 50 repeated simulations.

 \begin{figure}[htb]
 	\begin{lstlisting}[language=C++ ]
 	data {
 	int <lower=0> n;     //number of observations;
 	int <lower=0> d;      //number of predictor variables; 
 	matrix [n,d] x ;       // predictors; we vary its correlation during simulations.
 	int<lower=0,upper=1> y[n]; // binary outcome;
 	}
 	parameters {
 	vector[d] beta;
 	}
 	model {
 	y ~ bernoulli_logit(x*beta);
 	}
 	\end{lstlisting}\caption{\em Stan code for logistic regressions}\label{ls_logistic_reg}
 \end{figure}
 
 Figure 3 and 4 in the main paper is a simulation result of a logistic regression 
  $$Y \sim \mathrm{Bernoulli} \left( \mathrm{logit}^{-1}( \beta X  ) \right) $$ with a flat prior on $\beta$.  
  We vary the correlation in design matrix by generating $X$ from $\hbox{N}(0, (1-\rho) I_{d\times d}+\rho 1_{d\times d})$, where $1_{d\times d} $ represents the $d$ by $d$ matrix with all elements to be $1$. In this experiment we fix a small number $n=100$ and $d=2$ since the main focus is parameter correlations. We compare $\hat k$ with the log predictive density, which is calculated from 100 independent test data.  The true posterior is from NUTS  in Stan with 4 chains and 3000 iterations each chain.  The $\hat k$ estimation, IS and  PSIS adjustment are  calculated from $10^5$ posterior samples.  To account for the uncertainty,  $\hat k$, log predictive density, and RMSE  are the average of 50 repeated experiments.

 \subsection {Eight-School Model}
 
 The \emph{eight-school model}  is named after \citet[section 5.5]{gelman2013bayesian}. The study was performed for the Educational Testing Service to analyze the effects of a special coaching program on students' SAT-V (Scholastic Aptitude Test Verbal) scores in each of eight high schools.   The outcome variable in each study was the score of a standardized multiple choice test. Each school $i$ separately analyzed  the treatment effect and reported  the mean  $y_i$  and standard deviation of the treatment effect  estimation $\sigma_i$, as summarized in Table \ref{tab:8_school_data}.
 
 \begin {table} 	\begin{center} \small \tabcolsep=0.1cm
 	\begin{tabular}{ccc} 
 		School Index $j$  & Estimated  Treatment  	Effect $y_i$ & Standard Deviation   of Effect   
 		Estimate $\sigma_j$ \\ 
 		\hline 
 		1 & 28 & 15 \\ 
 		2 & 8 & 10 \\ 
 		3 & -3 & 16 \\ 
 		4 & 7 & 11 \\ 
 		5 & -1 & 9 \\ 
 		6 & 1 & 11 \\ 
 		7 & 8 & 10 \\ 
 		8 & 12 & 18 \\   
 	\end{tabular}
 \end{center} \vspace{-0.2 in}\caption{\em School-level observed effects of special preparation on SAT-V scores in eight randomized experiments. Estimates are based on separate analyses for the eight experiments.} \label{tab:8_school_data}
\end{table}

 There was no prior reason to believe that any of the eight programs was more effective than any other or that some were more similar in effect to each other than to any other. Hence, we view them as independent experiments and apply a Bayesian hierarchical normal model:
 \begin{align*}
 y_j|\theta_j   & \sim \mbox{N} (\theta_j, \sigma_j), \quad  \theta_j  \sim \mbox{N} (\mu, \tau), \quad 1 \leq j \leq 8, \\
 \mu &\sim \N(0,5),  \quad  \quad    \tau  \sim \mathrm{half\!\!-\!\!Cauchy}(0,5).
 \end{align*}
 where $\theta_j$ represents the underlying treatment effect in school $j$, while $\mu$ and  $\tau$ are the hyper-parameters that are shared across all schools.

  \begin{figure}[htb]
 	\begin{lstlisting}[language=C++ ]
 	data {
 	int<lower=0> J;      // number of schools
 	real y[J];         // estimated treatment
 	real<lower=0> sigma[J];     // std  of estimated effect
 	}
 	
 	parameters {
 	real theta[J];    // treatment effect in school j
 	real mu;    // hyper-parameter of mean
 	real<lower=0> tau;   // hyper-parameter of sdv
 	}
 	model {
 	theta ~ normal(mu, tau);
 	y ~ normal(theta, sigma);
 	mu ~ normal(0, 5);    // a non-informative prior
 	tau ~ cauchy(0, 5);
 	}
 	\end{lstlisting}\caption{\em Stan code for centered parametrization in the eight-school model. It  leads to strong dependency between $tau$ and $theta$.}\label{ls_school_cp}
 \end{figure}

 \begin{figure}[htb]
 \begin{lstlisting}[language=C++ ]
 data {
 int<lower=0> J;      // number of schools
 real y[J];         // estimated treatment
 real<lower=0> sigma[J];     // std  of estimated effect
 }
 parameters {
 vector[J] theta_trans;   // transformation of theta
 real mu;    // hyper-parameter of mean
 real<lower=0> tau;  // hyper-parameter of sd
 }
 transformed parameters{
 vector[J] theta;     // original theta
 theta=theta_trans*tau+mu;
 }
 model {
 theta_trans ~normal (0,1);
 y ~ normal(theta, sigma);
 mu ~ normal(0, 5);    // a non-informative prior
 tau ~ cauchy(0, 5);
 }
 \end{lstlisting}\caption{\em Stan code for non-centered parametrization in the eight-school model. It  extracts the dependency between $tau$ and $theta$.}\label{ls_school_ncp}
 \end{figure}

 There are two parametrization forms being discussed: \emph{centered parameterization} and  \emph{ non-centered parameterization}. Listing    \ref{ls_school_cp} and \ref{ls_school_ncp} give two Stan codes separately.
The true posterior is from NUTS  in Stan with 4 chains and 3000 iterations each chain.    The $\hat k$ estimation and  PSIS adjustment are  calculated from $S=10^5$ posterior samples. The marginal $\hat k$ is calculated by using the NUTS  density, which is typically unavailable for more complicated problems in practice.
 
The VSBC test in Figure  6 is based on $M=1000$ replications and we  pre-register the  first treatment effect $\theta_1$ and group-level standard error $\log \tau$ before the test. 
 
As discussed in Section 3.2, VSBC assesses the average calibration of the point estimation. Hence the result depends on the choice of prior. For example, if we instead set the prior to be
$$ \mu \sim \N(0,50),  \quad  \quad    \tau  \sim \mathrm{N}^{+}(0,25),  $$ 
which  is essentially flat in the region of interesting part of the likelihood and more in agreement with the prior knowledge, then the result of VSBC test change to Figure \ref{fig:school_VSBC_flat}. Again, the skewness of $p$-values verifies VI estimation of $\theta_1$ is in average unbiased  while $\tau$ is biased in both centered and non-centered parametrization.

 \begin{figure}[!h]
	\begin{center}
		\includegraphics[width=0.6\linewidth]{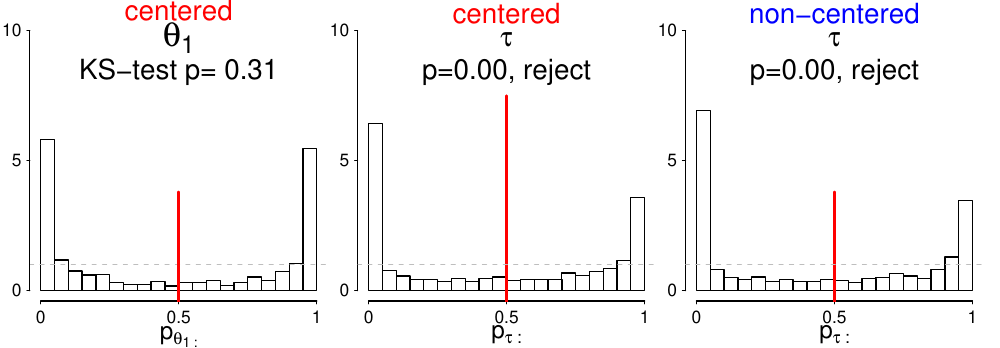}
	\end{center}\vspace{-0.15in}
	\caption{\em  The VSBC diagnostic of the eight-school example under a non-informative prior $\mu \sim \N(0,50), \,   \tau  \sim \mathrm{N}^{+}(0,25)$.  The skewness of $p$-values verifies VI estimation of $\theta_1$ is in average unbiased  while $\tau$ is biased in both centered and non-centered parametrization.} \label{fig:school_VSBC_flat} \vspace{-0.05in}
\end{figure}

 \subsection {Cancer Classification Using  Horseshoe Priors}
In Section 4.3 of the main paper we replicate the cancer classification under regularized horseshoe prior as first introduced by \citet{piironen2017sparsity}.

 The Leukemia microarray cancer classification dataset \footnote{The Leukemia classification dataset can be downloaded from  \url{http://featureselectiocn.asu.edu/datasets.php}}. 
 It contains  $n=72$ observations  and $d=7129$ features $X_{n\times d}$. $X$ is standardized before any further process.  The outcome $y_{1:n}$ is binary, so we can fit a logistic regression
 $$y_i |\beta \sim \mathrm{Bernoulli}\left(  \mathrm{logit}^{-1}\left(   \sum _{j=1}^d \beta_j x_{ij}  +\beta_0 \right )  \right).  $$
 There are far more predictors than observations, so we expect only a few of predictors to be related and therefore have a regression coefficient distinguishable from zero. Further, many predictors are correlated, making it necessary to have a regularization.
 
 To this end, we apply the \emph{regularized horseshoe prior}, which is a generalization of \emph{horseshoe prior}. 
\begin{align*}
& \beta_j|\tau, \lambda, c  \sim \N(0, \tau^2 \tilde \lambda_j^2) ,  \quad  \quad c^2 \sim \mathrm{Inv\!\!-\!\!Gamma} (2,8), \\
 & \lambda_j  \sim \mathrm{Half\!\!-\!\!Cauchy}(0,1),   \quad   \tau | \tau_0   \sim \mathrm{Half\!\!-\!\!Cauchy} (0,\tau_0).
\end{align*}
The scale of the global shrinkage is set according to the recommendation $\tau_0=   2 \left( n^{1/2} (d-1)\right) ^{-1}$ There is no reason to shrink intercept so we put $\beta_0 \sim \N(0,10)$. 
The Stan code is summarized  in Figure \ref{ls_HS}.

We first run NUTS in Stan with 4 chains and 3000 iterations each chain.  We manually pick $\beta_{1834}$, the coefficient that has the largest posterior mean. The posterior distribution of it is  bi-modal with one spike at 0.

ADVI is implemented using the same parametrization and we  decrease the learning rate $\eta$ to 0.1 and the threshold  \texttt{tol\_rel\_obj} to $0.001$

The $\hat k$ estimation is based on $S=10^4$ posterior samples. Since $\hat k$ is extremely large, indicating VI is far away from the true posterior and no adjustment will work,  we do not further conduct PSIS.

\begin{figure}[H]
	\begin{lstlisting}[language=C++ ]
	data {
	int<lower=0> n;				    // number of observations
	int<lower=0> d;             // number of predictors
	int<lower=0,upper=1> y[n];	// outputs
	matrix[n,d] x;				      // inputs
	real<lower=0> scale_icept;	// prior std for the intercept
	real<lower=0> scale_global;	// scale for the half-t prior for tau
	real<lower=0> slab_scale;
	real<lower=0> slab_df;
	}
	parameters {
	real beta0; // intercept
	vector[d] z; // auxiliary parameter
	real<lower=0> tau;			// global shrinkage parameter
	vector<lower=0>[d] lambda;	// local shrinkage parameter
	real<lower=0> caux; // auxiliary
	}
	transformed parameters {
	real<lower=0> c;
	vector[d] beta;				// regression coefficients
	vector[n] f;				// latent values
	vector<lower=0>[d] lambda_tilde;
	c = slab_scale * sqrt(caux);
	lambda_tilde = sqrt( c^2 * square(lambda) ./ (c^2 + tau^2* square(lambda)) );
	beta = z .* lambda_tilde*tau;
	f = beta0 + x*beta;
	}
	model {
	z ~ normal(0,1);
	lambda ~ cauchy(0,1);
	tau ~ cauchy(0, scale_global);
	caux ~ inv_gamma(0.5*slab_df, 0.5*slab_df);
	beta0 ~ normal(0,scale_icept);
	y ~ bernoulli_logit(f);
	} \end{lstlisting}\caption{\em Stan code for regularized horseshoe logistic regression.}\label{ls_HS}
\end{figure}

In the VSBC test, we pre-register that pre-chosen coefficient $\beta_{1834}$, $\log \lambda_{1834}$ and global shrinkage $\log \tau$ before the test. The VSBC diagnostic is based on M=1000 replications.

\end{document}